\documentclass{svproc}
\usepackage{url}

\usepackage{amsfonts}	
\usepackage{amssymb}
\usepackage{amsmath}
\let\proof\relax \let\endproof\relax 
\usepackage{amsthm}
\theoremstyle{plain}
\usepackage{xspace}
\usepackage{booktabs}

\def\cP{P}

\newenvironment{examplecont}[1]{\begin{example}[Ex.~\ref{#1} continued]}{\end{example}}

\newcounter{programcount}
\newcommand{\newprogram}{\refstepcounter{programcount}\ensuremath{\mathcal{P}_{\arabic{programcount}}}}

\newcommand{\program}[2]{\ensuremath{\mathcal{P}_{#1{#2}}}}

\spnewtheorem{spproposition}[theorem]{Proposition}{\Shrink\bfseries}{\itshape\Shrink}
\spnewtheorem{sptheorem}[theorem]{Theorem}{\Shrink\bfseries}{\itshape\Shrink}
\spnewtheorem{spcorollary}[theorem]{Corollary}{\Shrink\bfseries}{\itshape\Shrink}
\spnewtheorem{spconjecture}[theorem]{Conjecture}{\Shrink\bfseries}{\itshape\Shrink}
\spnewtheorem{splemma}[theorem]{Lemma}{\Shrink\bfseries}{\itshape\Shrink}
\spnewtheorem{spdefinition}[theorem]{Definition}{\Shrink\bfseries}{\itshape\Shrink}

\newcommand{\CalL}{{\cal L}}

\newcommand{\CalP}{{\cal P}}

\newcommand{\udf}{\mathsf{U}}
\newcommand{\true}{\mathbf{\top}}

\newcommand{\atoms}{\mathsf{at}}
\newcommand{\Atoms}{\mathsf{At}}

\newcommand{\Head}{\mbox{\em Head}}
\newcommand{\Body}{\mbox{\em Body}}
\newcommand{\ud}{\mathsf{undef}}
\newcommand{\Def}{\mathsf{def}}

\newcommand{\nlnot}{\mathsf{not}\,}

\newcommand{\Svl}{\Phi}

\newcommand{\Svlp}{\Phi_{\CalP}}

\newcommand{\Ab}{\mathit{ab}}

\newcommand{\eqdef}{%
  \mathrel{\vbox{\offinterlineskip\ialign{%
    \hfil##\hfil\cr%
    $\scriptscriptstyle\mathrm{def}$\cr%
    \noalign{\kern1pt}%
    $=$\cr%
    \noalign{\kern-0.1pt}%
}}}}

\newcounter{bar}
\DeclareRobustCommand{\nr}[1]{%
   \refstepcounter{bar}%
   \thebar\label{#1}}

\newcounter{barp}
\newcommand{\Nformulas}{\mbox{\textbf{N}-formulas}\xspace}
\newcommand{\Lformulas}{\mbox{\textbf{\L}-formulas}\xspace}

\newcommand{\Nformula}{\mbox{\textbf{N}-formula}\xspace}

\newcommand{\Ntheory}{\mbox{\textbf{N}-theory}\xspace}
\newcommand{\Ntheories}{\mbox{\textbf{N}-theories}\xspace}

\newcommand{\Ltheory}{\mbox{\textbf{\L}-theory}\xspace}

\newcommand{\GammaI}[1]{{\tilde{\textbf{N}}(#1)}}

\newcommand{\NP}{{\textbf{N}(\CalP)}}
\newcommand{\NwoP}{{\textbf{N}}}

\newcommand{\Llogic}{\mbox{\textbf{\L}-logic}\xspace}
\newcommand{\Nlogic}{\mbox{\textbf{N}-logic}\xspace}

\newcommand{\modelsLogic}{\models_{\text{\tiny \ensuremath{\CalL}}}}
\newcommand{\equivLogic}{\equiv_{\text{\tiny \ensuremath{\CalL}}}}

\newcommand{\impliesN}{\leftarrow_{\text{\rm\tiny CL}}}
\newcommand{\dimpliesN}{\leftrightarrow_{\text{\tiny CL}}}
\newcommand{\equivN}{\equiv_{\text{\tiny CL}}}
\newcommand{\modelsN}{\models_{\text{\tiny CL}}}

\newcommand{\impliesL}{\leftarrow_{\text{\rm \tiny \L}}}
\newcommand{\modelsL}{\models_{\text{\tiny \text{\rm \L}}}}
\newcommand{\equivL}{\equiv_{\text{\tiny \text{\rm \L}}}}
\newcommand{\dimpliesL}{\leftrightarrow_{\text{\rm \tiny \L}}}

\newcommand{\leftrightL}{\leftrightarrow_{\text{\tiny \L}}}

\newcommand{\tuple}[1]{\ensuremath{\langle #1 \rangle}}

\newcommand{\set}[1]{\ensuremath{\{#1\}}}
\newcommand{\setm}[2]{\ensuremath{\{\ #1\ \big|\ #2\ \}}}

\newcommand{\tranN}[1]{\ensuremath{\mathbf{N}(#1)}}
\newcommand{\tranL}[1]{\mbox{\ensuremath{\normalfont{\textbf{\L}}}}(#1)}

\newcommand{\reg}[1]{\ensuremath{\text{reg}(#1)}}

\newcommand{\wComp}[1]{\ensuremath{\textrm{\small wc}(#1)}}
\newcommand{\nComp}[1]{{\tilde{#1}}}

\newcommand{\Pdisj}[1]{\ensuremath{\textrm{\small srf}(#1)}}

\newcommand{\WComp}[1]{\wComp{#1}}

\def\wcmodel{\mbox{wc-model}\xspace}
\def\wcmodels{\mbox{wc-models}\xspace}
\def\wcsmpl{\mbox{wc-normal}\xspace}

\newcommand{\essay}{\mathit{e}}

\newcommand{\lib}{\ell}

\newcommand{\ess}{e}

\newcommand{\smpl}{e \leadsto \lib}
\newcommand{\rmpl}{t \leadsto \lib}
\newcommand{\sneg}{\sim}

\newcommand{\ASS}{ASP\xspace}
\renewcommand{\sneg}{\mathsf{\neg}}
\newcommand{\wneg}{\mathsf{\mathsf{not}}\,}

\newcommand{\commentem}[1]{
}

\begin{document}
\mainmatter              
\title{On the Relation between 
Weak Completion Semantics and Answer Set Semantics}
%
\titlerunning{On the Relation between WCS and \ASS}  
%
\author{Emmanuelle-Anna Dietz Saldanha$^1$ and Jorge Fandinno$^2$}
\authorrunning{} 
%
\tocauthor{}
\institute{
$^1$
Technische Universit\!at Dresden,
\\
$^2$ Technische Universit\!at Potsdam,
\\Germany
\\}

\maketitle              

\begin{abstract}
The Weak Completion Semantics (WCS) is a computational cognitive theory that has shown to be successful in modeling episodes 
of human reasoning. 
As the WCS is a recently developed logic programming approach, this paper investigates the correspondence of the WCS with respect to 
the well-established Answer Set Semantics (\ASS). 
The underlying three-valued logic of both semantics is different
and their models are evaluated with respect to different program transformations.
We first illustrate these differences by the formal representation of some examples of a well-known psychological experiment, the suppression task.
After that, we will provide a translation from logic programs understood under the WCS into logic programs understood under the \ASS. In particular, we will show that logic programs under the WCS can be represented as logic programs under the \ASS by means of 
a
\emph{definition completion}, where all \emph{defined} atoms in a program must be false when their definitions are false.
\end{abstract}

\section{Introduction}
\label{sec:introduction}

The Weak Completion Semantics (WCS), originally presented in~\cite{hk:2009a},
has been suggested as a computational cognitive theory, and demonstrated to be adequate for modeling various episodes of human reasoning 
summarized in~\cite{btg:2015:h}. 
Consider a well-known psychological experiment, the suppression task~\cite{byrne:89}, which showed that participants'
answer systematically diverged from classical logic correct answers. 
Participants were asked to derive conclusions given variations of a set of premises. 
The first group was given the following two premises:\footnote{The participants received only the natural language sentences, not the abbreviations.}
\begin{align}
&\textit{If she has an essay to finish, then she will study late in the library.} \label{e:2} \tag{$\smpl$}\\
&\textit{She does not have an essay to finish.} \label{e:11} \tag{$\nlnot {\essay}$}
\end{align}
Then, they were asked what necessarily follows assuming that the 
above premises were true
and given three possible answer from where
they could choose:
\begin{align}
&\textit{She will study late in the library.} \label{l:1} \tag{$\lib$}\\
&\textit{She will not study late in the library.} \label{l:2} \tag{$\nlnot \lib$}\\
&\textit{She may or may not study late in the library.} \label{l:3} \tag{$\lib\ \mathit{ or }\ \nlnot{\lib}$}
\end{align}
54\% of the participants answered that \emph{She will not study late in the library}.
The second group received, additionally to~(\ref{e:2}) and~(\ref{e:11}), the following
premise:
\begin{align}
&\textit{If she has a textbook to read, then she will study late in the library.} \label{e:21} \tag{$\rmpl$}
\end{align}
Now, only 4\% of the participants answered that \emph{She will not study late in the library}.
With these results, Byrne showed that humans seem to reason non-monotonically, i.e.\ they suppressed previously drawn conclusions.
The above examples demonstrates that humans do not always apply the \emph{close world assumption} in their inferences.
In particular, if they are \emph{made aware} of alternatives, they might rather apply the \emph{open world assumption}.
Stenning and van Lambalgen~\cite{stenning:vanlambalgen:2008} suggested a formal representation of these premises
by licenses for inferences. For the first group, the following logic program rules were suggested:
\begin{align}
\lib \leftarrow \essay \wedge \nlnot \Ab_1
\hspace{2cm}
\ess \leftarrow \bot
\hspace{2cm}
\Ab_1 \leftarrow \bot
  \label{prg:suppression1}
\end{align}
$\Ab_1$ is an abnormality predicate. $\Ab_1 \leftarrow \bot$ and $e \leftarrow \bot$ are (negative) \emph{assumptions} 
which assume $e$ and $\Ab_1$ being false.\footnote{The rules will be understood under their \emph{weak completion}.}
In the designated model under the WCS~\cite{hk:2009a}, which is the least model of the weak completion of that program 
under the three-valued {\L}ukasiewicz logic~\cite{lukasiewicz:20}, $e$, $\lib$ and $\Ab_1$ are false.
On the other hand, the logic program rules for the second group was suggested to be as follows:
\begin{align}
\lib \leftarrow \essay \wedge \nlnot \Ab_1
\hspace{0.6cm}
\lib \leftarrow t \wedge \nlnot \Ab_2
\hspace{0.6cm}
\ess \leftarrow \bot
\hspace{0.6cm}
\Ab_1 \leftarrow \bot
\hspace{0.6cm}
\Ab_2 \leftarrow \bot
  \label{prg:suppression2}
\end{align}
Here, in the designated model under the WCS, $e$, $\Ab_1$ and $\Ab_2$ are false, and $\lib$ and $t$ \emph{unknown}.
The differences between both models, where in the first case $\lib$ is false, and in the second case $\lib$
is unknown, seems to represent well the \textit{suppression effect} occurring in the second group: In
the first group, 46\% concluded that \emph{She will not study late in the library},
whereas in the second group, only 4\% concluded that \emph{She will not study late in the library}.
The overall results of all the twelve cases of the suppression tasks seem to be adequately modeled under
the WCS~\cite{cogsci:2012}.

In this paper we will investigate how the above two cases of the suppression task can be modeled under the Answer Set Semantics~\cite{gelfond:lifschitz:91} (\ASS), in particular how both semantics correspond to each other.
For this purpose, we first introduce the notions and notation used throughout the paper and the underlying three-valued logics.
Section~\ref{sec:logicprogramming} introduces \ASS and WCS and shows some intermediate results.
The main result is presented in Section~\ref{sec:correspondence}, where the formal correspondence between both semantics is shown,
and the above two cases will be discussed again.

\section{Preliminaries}
\label{sec:preliminaries}

In this section, we present the general notation and terminology that will
be used throughout the paper together with the semantics for \emph{classical logic with strong negation}~\cite{vakarelov1977notes} and \emph{three-valued {\L}ukasiewicz logic}~\cite{lukasiewicz:20}.
In the sequel, definitions are specified in the running text, except if we intend to emphasize them.

\subsection{Syntax}
\label{sec:preliminaries.syntax}

We assume a fixed non-empty and (possibly infinite) set of ground atoms, denoted by $\Atoms$.
The set of (strongly) negated atoms for the atoms in $S \subseteq \Atoms$, is defined as $\sneg S \eqdef \{ \sneg A \mid A \in S \}$. 
A \emph{literal} $L$ is either an atom or its (strong) negation, 
that is \mbox{$L \in (\Atoms \ \cup \sneg \Atoms)$}.
\commentem{why is there a need to understand truth constants as literals? we will have to treat them as literals throughout the paper then, e.g. concerning the
definitions for $\Head(r)$ and $\Body(r)$, and possibly somewhere else...}
Given a set of atoms $\Atoms$, a \emph{formula} is defined according to the following grammar:
\[
\varphi ::= 
A \ \mid \ \bot
  \ \mid \ \top
  \ \mid \ \udf
  \ \mid \ \varphi \circ \psi
  \ \mid \ \sneg \varphi
  \ \mid \ \wneg\varphi
\]
$\top$,~$\bot$ and~$\udf$ denote the truth constants \textit{true}, \textit{false}
and \textit{unknown}, respectively.
The connective~``$\wneg$'' stands for \emph{weak} or \emph{default negation}, whereas ``$\sneg$'' stands for \emph{strong negation}. 
The
connectives~``$\impliesN$'', ``$\impliesL$'' and ``$\leftarrow$'' stand for classical (or material) implication, {\L}ukasiewicz implication and logic programming implication,
respectively. The logic programming implication sign $\leftarrow$ is purely syntactic and,
different to $\impliesN$ and $\impliesL$, will not be assigned a fixed underlying semantics. We will study two different logic programming semantics and depending on the semantics in consideration,
the meaning for $\leftarrow$ is then specified accordingly.
We use~$\leftrightarrow_{\text{\rm \tiny X}}$ as an abbreviation defined by
\begin{gather}
\varphi \leftrightarrow_{\text{\rm \tiny X}} \psi  \ \ \eqdef  \ \ (\varphi \leftarrow_{\text{\rm \tiny X}} \psi) \wedge (\psi\leftarrow_{\text{\rm \tiny X}} \varphi),
\end{gather}
where $\text{X} \in \{ \text{C}, \text{\L} \}$. 
A formula~$\varphi$ is \emph{regular} if 
its only occurrences of implications 
strong negation~$\sneg$ only occurs in front of atoms.
A formula~$\varphi$ is called \emph{implication-free} if there are no occurrences of the implication connectives $\impliesN$, $\impliesL$ or $\leftarrow$,
i.e.\ its set of connectives is $\{\wedge,\vee, \sneg, \wneg\}$.
A formula~$\varphi$ is called \emph{basic} if it is implication-free, and in addiction, it has no occurrences of weak negation, i.e. its only connective are~$\set{\wedge,\vee,\sneg\,}$.
Basic formulas are implication-free formulas, but not vice versa.
Note that, in general, regular formulas 
that are implication-free need not to be basic nor basic formulas need to be regular.

\begin{example}
Consider the following three formulas:
\begin{center}
$\varphi_{\nr{implicationfree}} = (\sneg (\sneg p \wedge q) \vee \wneg r)$\quad\quad\quad\quad
$\varphi_{\nr{basic}} = (\sneg(\sneg p \wedge q))$\quad\quad\quad\quad
$\varphi_{\nr{regular}} = (\wneg(\sneg p \wedge q))$
\end{center}
$\varphi_{\ref{implicationfree}}$, $\varphi_{\ref{basic}}$
and $\varphi_{\ref{regular}}$ are implication-free formulas, but
$\varphi_{\ref{implicationfree}}$ is neither basic nor regular.
$\varphi_{\ref{basic}}$ is basic, but not regular,
whereas $\varphi_{\ref{regular}}$ is regular but not basic.
\end{example}

\subsection{Three-valued Semantics}\label{sub:threevaluedsem}
 
\begin{table}[t]  
\[
\begin{array}[c]{@{\hspace{0mm}}c@{\hspace{2mm}}c}
F & \wneg F\\ \midrule
\top & \bot \smallskip \\
\udf & \top \\
\bot & \top \smallskip 
\end{array}
\quad\quad\quad\quad
\begin{array}{c@{\hspace{2mm}}lll}
 \wedge & \top & \udf & \bot \smallskip \\
\midrule
\top & \top & \udf & \bot \smallskip \\
\udf & \udf & \udf & \bot \smallskip \\
\bot & \bot & \bot & \bot 
 \end{array}
\quad\quad\quad\quad
\begin{array}{c@{\hspace{2mm}}lll} 
\impliesN & \top & \udf & \bot \smallskip \\
\midrule
\top & \top & \top & \top \smallskip \\
\udf & \udf & \top & \top \smallskip \\
\bot & \bot & \top & \top 
 \end{array}
\quad\quad\quad\quad
\begin{array}{c@{\hspace{2mm}}lll} 
\dimpliesN & \top & \udf & \bot \smallskip \\
\midrule
\top & \top & \udf & \bot \smallskip \\
\udf & \udf & \top & \top \smallskip \\
\bot & \bot & \top & \top 
 \end{array}
\]
 \[
\begin{array}[c]{@{\hspace{0mm}}c@{\hspace{2mm}}c}
F & \sneg F\\ \midrule
\top & \bot \smallskip \\
\udf & \udf \\
\bot & \top \smallskip 
\end{array}
\quad\quad\quad\quad
\begin{array}{c@{\hspace{2mm}}lll}
 \vee & \top & \udf & \bot \smallskip \\
\midrule
\top & \top & \top & \top \smallskip \\
\udf & \top & \udf & \udf \smallskip \\
\bot & \top & \udf & \bot 
 \end{array}
\quad\quad\quad\quad
\begin{array}{c@{\hspace{2mm}}lll} 
\impliesL & \top & \udf & \bot \smallskip \\
\midrule
\top & \top & \top & \top \smallskip \\
\udf & \udf & \top & \top \smallskip \\
\bot & \bot & \udf & \top 
 \end{array}
\quad\quad\quad\quad
\begin{array}{c@{\hspace{2mm}}ccc}
\leftrightL & \top & \udf & \bot \smallskip  \\
\midrule
\top & \top & \udf & \bot \smallskip \\
\udf & \udf & \top & \udf \smallskip \\
\bot & \bot & \udf & \top 
 \end{array}
 \]
\caption{
Truth tables for three-valued {\L}ukasiewicz logic $\{ \sneg, \wedge, \vee, \impliesL, \leftrightL\}$,
and Classical Logic extended with strong negation $\{ \sneg, \wneg, \wedge, \vee, \impliesN, \dimpliesN \}$.
\label{tab:3vld}}
\vspace{-0.5cm}
\end{table}
A (three-valued) interpretation 
 $I : \Atoms \longrightarrow \{ \top, \bot, \udf\}$ is a function mapping each atom to a truth constant.
We introduce now two different alternative representations of interpretations that are equivalent and common in the literature:
An interpretation~$I$ can be represented as a pair of set of atoms, 
$I = \tuple{I^\top, I^\bot}$ such that $I^\top \cap I^\bot = \emptyset$
where the correspondence is given as follows:
$I^\top = \{ A \mid I(A) = \top \}$ and $I^\bot = \{ A \mid I(A) = \bot \}$.
Note that $A \notin (I^\top \cup I^\bot)$ holds iff $I(A) = \udf$.
Alternatively, $I$ can be represented as a set of literals $I^\top \cup \sneg I^\bot$.
The first representation is usual in the context of the WCS while the second is usual in the context of Answer Set Programming (ASP)~\cite{Baral:2003}.
We will use them interchangeably.

Three-valued interpretations can be ordered either by \emph{knowledge} or by \emph{truth}: Given two interpretations \mbox{$I=\tuple{I^\top,I^\bot}$}
and \mbox{$J=\tuple{J^\top,J^\bot}$},
we say that $I$ \emph{contains less knowledge} than $J$,
in symbols $I \subseteq_k J$,
iff
$I^\top \subseteq J^\top$ and $I^\bot \subseteq J^\bot$
iff $(I^\top \cup \sneg I^\bot) \subseteq (J^\top \cup \sneg J^\bot)$.
In other words, $I$ and $J$ agree in all atoms which are known in $I$, but $I$ can have more unknown atoms.
On the other hand, when the \emph{truth order} is applied, i.e.\ $\bot \leq \udf \leq \top$, then, given two interpretations \mbox{$I=\tuple{I^\top,I^\bot}$}
and \mbox{$J=\tuple{J^\top,J^\bot}$}, we say that $I$ \emph{contains less truth} than $J$,
in symbols $I \subseteq_t J$,
iff
$I^\top \subseteq J^\top$ and $J^\bot \subseteq I^\bot$
iff $(I^\top \cup \sneg J^\bot) \subseteq (J^\top \cup \sneg I^\bot)$.
As we are only interested in the knowledge ordering, we will omit the subscript $k$ in the following, and simply write $I \subseteq J$ when we refer to $I \subseteq_k J$.

In this paper, we will consider two different three-valued logics, so we introduce some general definitions parametrized by the logic.
Given a (three-valued) logic~$\CalL$, a three-valued interpretation~$I$ \emph{satisfies} a formula~$\varphi$, in symbols $I \modelsLogic \varphi$, iff
$I$ evaluates $\varphi$ as true, that is $I(\varphi) = \top$.
Furthermore, $I$ is called a \emph{(\mbox{three-valued}) model} of a theory 
$\Gamma$ (where $\Gamma$ is a set of formulas) under $\CalL$, denoted by $I \modelsLogic \Gamma$,
iff $I \modelsLogic \varphi$ for all $\varphi \in \Gamma$.
$I$ is a \emph{$\subseteq$-minimal model} of $\Gamma$ iff for no other model $J$ of $\Gamma$,
$J \subset I$ (ordered according to the knowledge).
$I$ is the \emph{$\subseteq$-least model} of $\Gamma$ iff it is the unique
minimal model of $\Gamma$.
A formula $\varphi$ is \emph{valid} in $\CalL$, denoted by $\modelsLogic \varphi$,
 iff $I \modelsLogic \varphi$ for every interpretation~$I$. 
Furthermore, we write $\Gamma \modelsLogic \varphi$ iff 
$I \modelsLogic \Gamma$ implies $I \modelsLogic \varphi$ for every interpretation~$I$.
For theories $\Gamma$ and $\Gamma'$,
we write $\Gamma \equivLogic \Gamma'$ iff 
every interpretation~$I$
satisfies: $I \modelsLogic \Gamma$ iff $I \modelsLogic \Gamma'$.
We will omit the brackets~$\{$ and~$\}$, in case $\equivLogic$ is applied to formulas, i.e.\
we write $\varphi \equivLogic \varphi'$ iff 
$\set{\varphi} \equivLogic \set{\varphi'}$.


\subsection{Classical Logic extended with strong Negation (\Nlogic)}
\label{sec:preliminaries.nlogic}

The distinction between strong and weak negation was first noticed by Nelson~\cite{nelson:1949}
in the context of Intuitionistic Logic and later studied by Vakarelov~\cite{vakarelov1977notes} in the context of Classical Logic.
The syntax is obtained from the syntax of Classical Logic extended with a the connective ``$\sneg\,$'' standing for strong negation, that is, formulas are built from the set of connectives~$\{ \wedge, \vee, \impliesN, \wneg, \sneg\, \}$.
Note that here classical negation is dentoed by ``$\wneg\!$''.
We call them the \Nformulas.
Accordingly, a \Ntheory is a set of \Nformulas.
Evaluation of its connectives, $\wedge, \vee, \impliesN, \wneg$ and $\sneg$, is given by the corresponding truth tables in Table~\ref{tab:3vld}.  In the sequel, we refer to \Nlogic if \Nformulas or \Ntheories are considered and evaluated with respect to these truth tables.
Note that \Nlogic is a conservative extension of classical logic in the sense that, if we restrict ourselves to formulas without strong negation (formulas without the $\sneg$ connective), the valid formulas in \Nlogic and Classical logic are the same.
We use $\modelsN$ and $\equivN$ to denote entailment and equivalence according to \Nlogic.
Weak negation~``$\wneg$'' can be defined in terms of ``$\impliesN$'' by the following equivalence:
\begin{align}
\wneg\varphi \ \ \equivN \ \ \bot \impliesN \varphi
    \label{vorob:impliesneg}
\end{align}
We will consider weak negation here as connective in its own right because of its importance for logic programming.

It is interesting to note that, every (possibly non-regular) \Nformula can be rewritten as as an equivalent regular \Nformula applying the following equivalences
taken from Vorob'ev calculus (see Section 2.1 in~\cite{pearce:1997} for more details):
\begin{align}
\sneg\sneg \varphi & \equivN \varphi & 
	\label{vorob:atom}\\
\sneg\wneg\varphi &\equivN \varphi
\label{vorob:neg} \\
\sneg\,(\varphi \wedge \psi) &\equivN \sneg\varphi \vee \sneg\psi
\label{vorob:wedge}\\
\sneg\,(\varphi \vee \psi) &\equivN \sneg\varphi \wedge \sneg\psi  \label{vorob:vee}\\
\sneg\,(\varphi \impliesN \psi) & \equivN \sneg\varphi \wedge \psi  \label{vorob:negimplies}
\end{align}
For any \Nformula~$\varphi$, we write $\reg{\varphi}$ for the regular formula obtained from $\varphi$ by applying the above equivalences.
For a theory $\Gamma$, by $\reg{\Gamma} \eqdef \setm{ \reg{\varphi} }{ \varphi \in \Gamma }$ we denote the regular theory obtained in the same way.

\subsection{Three-valued {\L}ukasiewicz Logic (\Llogic)}
\label{sec:preliminaries.llogic}

The syntax of the three-valued {\L}ukasiewicz logic introduced in~\cite{lukasiewicz:20} is restricted to the set of connectives $\{ \bot, \wedge, \vee, \impliesL, \sneg\, \}$,
that is, by replacing in Classical Logic connectives ``$\impliesN$'' and ``$\wneg$'' by ``$\impliesL$'' and ``$\sneg\,$'', respectively.
We call formulas build from these connectives \Lformulas.
Accordingly, a \Ltheory is a set of \Lformulas.
Evaluation of its connectives, $\bot, \wedge, \vee, \impliesL$ and $\sneg$,
is given by the corresponding truth tables in Table~\ref{tab:3vld}. In the sequel, we refer to the \Llogic, if \Lformulas or a \Ltheory are considered and evaluated with respect to these truth tables.
We use $\modelsL$ and $\equivL$ to denote entailment and equivalence according to \Llogic.

\section{Logic Programming}\label{sec:logicprogramming}

\commentem{In the following definitions i have some need for clarification for the following terms:
\begin{enumerate}
\item a basic rule, is never a basic formula. Right?
\item a rule is regular iff its body and head are regular: however regular 
formulas can have implication, i.e. $p \leftarrow q \leftarrow q \leftarrow s$ is 
a regular rule? head and body can according to Vorob'ev calculus be transformed into 
implication-free formulas, but only if the implication sign is classical, i.e. shouldn't we restrict
regular formulas to be formulas, that, only contain the 
$\impliesN$ implication but not the $\leftarrow$ implication, right? i.e. the above formula
should then be rewritten as $p \impliesN q \leftarrow q \impliesN s$
or $p \leftarrow q \impliesN q \impliesN s$ and so we know exactly what is head and what is body.
\item how to write facts in logic programs? I.e. $p$ to denote that $p$ is true
and $\sneg p$ to denote that $p$ is false. However this does not correspond to the specification for
rules in logic programs which need to have the form $\varphi \leftarrow \psi$. Further, how how these
 facts evaluated according to the definition for closed interpretations? (i know how they are meant to be evaluated, however
 we never discuss this formally in the text)
\end{enumerate}
}
A rule $r$ is an expression of the form:
\begin{align}
\varphi \leftarrow \psi  \label{eq:rule}
\end{align}
where
$\varphi$ and $\psi$ are \mbox{implication-free} formulas respectively called the \emph{head} the \emph{body} of the rule and $L$ is a literal. $\Head(r)$ and $\Body(r)$ denote the set of literals that occur in the head and the body of the rule $r$, respectively.
A rule is \emph{basic} (resp. \emph{regular}) iff its body and head are basic (resp. regular).
A \emph{normal nested rule} iff is a rule whose head $\varphi \in (\Atoms \cup \neg\Atoms)$ is a literal, 
an \emph{extended rule} is a normal nested rules such that its body is either $\top$, (this rule is called a positive fact), or a conjunction of literals or literals preced by weak negation.
A \emph{normal rule} is an extended nested rules such that its head is a positive literal.

A \emph{(logic) program}~$\CalP$ is a set of rules
and, for any program~$\CalP$, by $\atoms(\CalP)$, we denote the set of all
atoms occurring in program~$\CalP$.
If it is clear from the context,
then we assume that $\Atoms = \atoms(\CalP)$.
We also denote by
$\Head(\CalP)= \{ L \mid L \in \Head(r) \mbox{ and } r \in \CalP \}$
and
$\Body(\CalP)= \{ L \mid L \in \Body(r) \mbox{ and } r \in \CalP \}$
the set of literals occurring in the head and the body of a program~$\CalP$, respectively.

A program is called \emph{basic} (resp. \emph{regular}, \emph{normal}, \emph{extended} or \emph{normal nested}) iff all its rules are basic (resp. regular, normal, extended or normal nested).
In case that $\psi = \top$, we will usually write just $\varphi$ instead of $\varphi \leftarrow \top$.
We will use $\varphi \leftrightarrow \psi$ as an abbreviation for a pair of rules
$\varphi \leftarrow \psi$ and $\psi \leftarrow \varphi$.
For instance, the program $\set{ a \leftrightarrow b }$ is an abbreviation for the program
$\set{ a \leftarrow b , b \leftarrow a }$.


\subsection{Answer Set Semantics}

We review now and extend the definition of Answer Set Semantics from~\cite{LifschitzTT99} to possibly non-regular programs.
The definition of the Answer Set Semantics to non-regular programs was first introduced by Pearce~\cite{pearce:1997} using an equilibrium condition over the models of the logic of \mbox{here-and-there} with strong negation.
The formulation we present here is equivalent, but using classical logic with strong negation and a reduct instead.

An interpretation~$I$ is said to be \emph{closed} under some regular program~$\CalP$
iff every 
\mbox{$(\varphi \leftarrow\psi) \in \CalP$} satisfies that
\mbox{$I \modelsN \varphi$}
whenever
\mbox{$I \modelsN \psi$} holds.
The \emph{reduct} of a regular formula~$\varphi$ with respect to an interpretation $I$, in symbols~$\varphi^I$, is recursively defined as follows:
\begin{align}
\varphi^I &\eqdef \varphi \text{ if $\varphi$ is a literal}
\\
(\varphi \wedge \psi)^I &\eqdef \varphi^I \wedge \psi^I
\\
(\varphi \vee \psi)^I &\eqdef \varphi^I \vee \psi^I
\\
(\wneg\varphi)^I &\eqdef \begin{cases}
\bot &\text{if } I \modelsN \varphi^I
\\
\top &\text{otherwise}
\end{cases}
\end{align}
The reduct of a rule $(\varphi \leftarrow \psi)^I \eqdef \varphi^I \leftarrow \psi^I$ is obtained by applying the reduct to its head and body.
The reduct of a logic program is obtained by applying the reduct to all its rules,
that is \ $\CalP^I \eqdef \setm{ r^I }{ r \in \CalP }$.
Furthermore, as done for \Ntheories, we also assign to any program $\CalP$ an equivalent regular program 
$\reg{\CalP} \eqdef \setm{ \reg{r} }{ r \in \CalP }$
with
$\reg{\varphi \leftarrow \psi} \eqdef \reg{\varphi} \leftarrow \reg{\psi}$ for any rule of the form
\mbox{$r = (\varphi \leftarrow \psi)$}.
 For any \Nformula~$\varphi$, we obtain the regular formula~$\reg{\varphi}$ by applying the equivalences (\ref{vorob:atom}-\ref{vorob:vee})
specified in Section~\ref{sec:preliminaries.nlogic}.
Then, answer sets are defined in terms of the regular counterpart of any program.
\commentem{The reduct of a program is also a program, right?}

\begin{definition}\label{def:answer-set}
Given an interpretation $I$ and a program $\CalP$, $I$ is an \emph{answer set} of~$\CalP$ iff $I$ is an
\mbox{$\subseteq$-minimal} closed interpretation under $\reg{\CalP}^I$.
\end{definition}

It is easy to see that, for regular programs, Definition~\ref{def:answer-set} precisely coincide with the definition of answer set from~\cite{LifschitzTT99}.

%

It is well-known that every answer set of a logic program without strong negation is also a  model in classical logic of the propositional theory obtained by replacing the logic programming implication~$\leftarrow$ by classical implication~$\impliesN$.
\commentem{reference}
We extend this result to the case of logic programs with strong negation by replacing classical logic by its extension with strong negation.
Formally, given a logic program~$\CalP$, by $\NwoP(\CalP)$ we denote the \Ntheory resulting of replacing in $\CalP$ each occurrence of~$\leftarrow$ by~$\impliesN$.

\begin{proposition}\label{prop:closedNmodel}
Given an interpretation~$I$ and a program~$\CalP$,
$I$ is closed under~$\CalP$ if and only if $I$ is a  model of $\NwoP(\CalP)$ under \Nlogic.
\end{proposition}

\begin{proof}
Note that $I$ is closed under~$\CalP$ iff all (rules) $\varphi \leftarrow \psi\in\CalP$ satisfy
\mbox{$I \modelsN \varphi$} whenever \mbox{$I \modelsN \psi$}
iff
all $\varphi \leftarrow \psi\in\CalP$ satisfy
either
\mbox{$I(\varphi)=\top$}  or \mbox{$I(\psi) \neq \top$}
iff
all $\varphi \leftarrow \psi\in\CalP$ satisfy
\mbox{$I(\varphi \impliesN \psi)=\top$}
iff
all formulas of the form $\varphi \impliesN \psi\in\NwoP(\CalP)$ satisfy
\mbox{$I(\varphi \impliesN \psi)=\top$}
iff $I$ is a model of $\NwoP(\CalP)$ under \Nlogic.
\end{proof}

\commentem{should the following corollary not rather be a proposition?}

\begin{corollary}\label{cor:closedNmodel}
Given an interpretation~$I$ and a regular program~$\CalP$,
$I$ is an answer set of~${\CalP}$ if and only if $I$ is a  $\subseteq$-minimal model of 
$\NwoP(\CalP^I)$ under \Nlogic.
\end{corollary}
\begin{proof}
From Definition~\ref{def:answer-set} and Proposition~\ref{prop:closedNmodel},
it follows that
$I$ is an answer set of~${\CalP}$ if and only if $I$ is a $\subseteq$-minimal model of 
$\NwoP({\CalP}^I)$ under \Nlogic.
Note that, since $\CalP$ is regular $\reg{\CalP} = \CalP$.
Furthermore, we also have $\NwoP(\reg{\CalP'}) = \reg{\NwoP(\CalP')} \equivN \NwoP(\CalP')$ for every program~$\CalP'$.
Hence, the statement holds.
\end{proof}
\commentem{About the last sentence of the proof above: This doesn't belong there, right? we do not discuss every program, but in the statement
we restrict already to regular programs. i we need the extension to any program, i would rather add a corollary for every program following from this proposition, i.e.
\begin{corollary}\label{cor:closedNmodel:forevery}
Given an interpretation~$I$ and a program~$\CalP$,
$I$ is an answer set of~${\CalP}$ if and only if $I$ is a  $\subseteq$-minimal model of 
$\NwoP(\CalP^I)$ under \Nlogic.
\end{corollary}
\begin{proof}
This follows immediately from Proposition~\ref{cor:closedNmodel} and 
that $\NwoP(\reg{\CalP'}) = \reg{\NwoP(\CalP')} \equivN \NwoP(\CalP')$ for every program~$\CalP'$. 
\end{proof}
}



\begin{proposition}\label{prop:answer.set->closedNmodel.basic}
Given an interpretation~$I$ and a regular program~$\CalP$,
if $I$ is an answer set of~${\CalP}$,
then $I$ is a model of $\NwoP({\CalP})$ under \Nlogic.
\end{proposition}

\commentem{
Could we show an intermediate result that 
\begin{center}
 $I$ is a model of $\tranN{\CalP}$ under \Nlogic if and only if $I$ is a model of $\tranN{\CalP^I}$ under \Nlogic
\end{center}
?
If yes, then we could simplify the proof as follows:
\begin{proof}
Assume that $I$ is an answer set of $\CalP$. According to Proposition~\ref{cor:closedNmodel}
$I$ is a  $\subseteq$-minimal model of 
$\tranN(\CalP^I)$ under \Nlogic, thus $I$ is a model of $\tranN(\CalP^I)$ under \Nlogic.
According to `the intermediate result above', $I$ is a model of $\tranN(\CalP)$ under \Nlogic.
\end{proof}
}

\begin{proof}[Proof (sketch)]
The proof can be carried out by induction in the structure of the formula by noting that $\varphi^I$ is just the result of partially evaluating subformulas for the form of $\wneg\psi$ with respect to $I$.
Hence,
from Proposition~\ref{prop:closedNmodel},
it follows
$I$ being an answer set of $\CalP$ implies that
$I$ is a ($\subseteq$-minimal) model of~$\tranN{\reg{\CalP}^I}$.
Take any formula of the form of $\varphi \impliesN \psi$ in $\CalP$.
Then, the formula $\reg{\varphi}^I \impliesN \reg{\psi}^I$ belongs to~$\tranN{\reg{\CalP}^I}$
and, since $I$ is a model of $\tranN{\reg{\CalP}^I}$,
we get that $I(\reg{\varphi}^I \impliesN \reg{\psi}^I) = \top$.
This implies that either $I(\reg{\varphi}^I) = \top$ or $I(\reg{\psi}^I) \neq \top$, which in its turn implies that
either $I(\varphi) = \top$ or $I(\psi) \neq \top$
holds.
Consequently, $I(\varphi \impliesN \psi) = \top$
and $I$ is a model of $\tranN{\CalP}$.
\end{proof}

As may be expected and as the following example shows, the other direction of Proposition~\ref{prop:answer.set->closedNmodel.basic} does not hold.
\begin{example}
 $\tuple{\emptyset,\emptyset} = \emptyset$ is the unique answer set of
$\CalP = \{ \sneg p \leftarrow q\}$
while its corresponding \Ntheory
\mbox{$\NP = \set { \sneg p \impliesN q }$} has several other models, as for instance $\tuple{\set{p},\set{q}} = \set{p,\neg q}$.
\end{example}

It is also well known that normal nested programs without weak nor strong negation (usually called \emph{positive}) have a unique answer set which coincides with the
\mbox{$\subseteq$-minimal} classical model (usually called the $\subseteq$-least model) of the corresponding propositional theory (see Proposition~3 in~\cite{Baral:2003}).
Similarly, normal nested programs without weak negation (i.e. basic) have at most one \mbox{$\subseteq$-minimal} model (and possibly no model) in classical logic with strong negation.
For instance, in the case of $\set{a,\sneg a}$ is inconsistent, so it has no $\subseteq$-least model,

\begin{proposition}\label{prop:unique.answer.set<->closedleastNmodel}
Given any basic normal nested program~$\CalP$,
one of the following two statements holds:\footnote{Recall, basic programs only consist 
of rules, whose head and body are implication-free and have no occurences
of weak negation.}
\begin{enumerate}
\item $\CalP$ has a unique answer set which is also the $\subseteq$-least model of~$\tranN{\CalP}$, or
\item $\CalP$ has no answer set and $\tranN{\CalP}$ has no model.
\end{enumerate}
\end{proposition}
\begin{proof}
First, note that if $\tranN{\CalP}$ has no model, from Proposition~\ref{prop:answer.set->closedNmodel.basic}, we immediately get that $\CalP$ has no answer set.
Let us show now that if $\tranN{\CalP}$ has a model, then it has a $\subseteq$-least model.
Obviously, since $\Atoms$ is finite, if $\tranN{\CalP}$ has a model, then it has at least some $\subseteq$-minimal model.
Suppose, for the sake of contradiction, that $\tranN{\CalP}$ has two different $\subseteq$-minimal models $I_1$ and $I_2$.
Then, there are literals $L_1$ and $L_2$ such that $I_1(L_1) = I_2(L_2) = \top$ and
$I_1(L_2) = I_2(L_1) = \bot$.
Let $J$ be an interpretation such that $J(L) = I_1(L)$ if $I_1(L) = I_2(L)$, and $J(L) = \udf$, otherwise.
Then, we have that $J \subset I_1$ and $J \subset I_2$ so $J \not\modelsN \tranN{\CalP}$.
Hence, there is a formula in $\tranN{\CalP}$ is of the form $L \impliesN \varphi$ with $L$ a literal and $\varphi$ a basic formula
such that $J(L) \neq \top$ and $J(\varphi) = \top$.
Furthermore, it can be shown by induction in the structure of the formula that, for every basic formula $\varphi$ and pair of interpretations $J \subseteq I$, we have that $J(\varphi) = \top$ implies $I(\varphi)=\top$.
Hence, we have that $I_1(\varphi)=\top$ and $I_2(\varphi)=\top$ which, since $I_1$ and $I_2$ are models of $\tranN{\CalP}$, implies that $I_1(L)=I_2(L) = \top$.
Hence, by construction we have that $J(L) =\top$,
which is a contradiction with the fact that $J(L) \neq \top$.
Consequently, there is a unique $\subseteq$-minimal model~$I$.
Finally, note that, since $\CalP$ has no weak negation, $\CalP^I=\CalP$ and, thus, $I$ is also the unique answer set of $\CalP$.
\end{proof}

\subsection{Weak Completion Semantics}

Formally, given a program~$\CalP$, by $\tranL{\CalP}$ we denote the \Ltheory resulting of replacing in $\CalP$ each occurrence of~$\leftarrow$ by~$\impliesL$.
Furthermore, for a normal nested program~$\CalP$,
we say that a literal $L$ is \emph{defined in $\CalP$} iff $\CalP$ contains a rule 
whose head is $L$; otherwise we say that $L$ is \textit{undefined in $\CalP$}.
The set of rules defining a literal $L$ (those with $L$ in the head) is denoted as $\Def(\CalP,L)$.
The set of all literals
that are undefined in~$\CalP$ is denoted by~$\ud(\CalP)$.
We specify the set of defined literals in $\CalP$ as $\Head(\CalP) = (\Atoms \cup \sneg \Atoms) \setminus \ud(\CalP)$.
The set $\Pdisj{\CalP}$ of a normal nested program~$\CalP$, is defined as follows:
\[\begin{array}{lll}
\Pdisj{\CalP}
\ \ \eqdef\ \{  L \leftarrow \varphi \vee \dots \vee \varphi_n \mid & 
L \in (\Atoms \cup \sneg\Atoms) \mbox{ and } \\ &
\Def(\CalP,L) = 
\set{ L \leftarrow \varphi, \dotsc, L \leftarrow \varphi_n } \not\emptyset \}.
\end{array}
\]
Below we straightforwardly extend the definition of the \emph{weak completion}~\cite{hk:2009a} to 
normal nested programs, that is, to programs 
that may contain rules where the head is a strongly negated literal.
The \emph{weak completion} of a normal nested program~$\CalP$,
denoted $\WComp{\CalP}$,
is defined as follows:
\[\begin{array}{lll}
\WComp{\CalP}
\ \ \eqdef\ \{  L \leftrightarrow \varphi \mid & 
 (L \leftarrow \varphi) \in \Pdisj{\CalP} \}
\end{array}
\]
Note that $\WComp{\CalP}$ is also a normal nested program because we consider that $\varphi \leftrightarrow \psi$ in a program is a shorthand for the two rules $\varphi \leftarrow \psi$ and $\psi \leftarrow \varphi$.

\begin{definition}
An interpretation~$I$ is called \emph{\wcmodel} of a normal nested program~$\cP$ iff $I$ is a $\subseteq$-minimal model of $\tranL{\WComp{\CalP}}$.
\end{definition}

Originally WCS was only defined for basic normal programs, extended with rules of the form $A \leftarrow \bot$, called (negative) \emph{assumption}~\cite{hk:2009a}. Here, we will call these programs, \wcsmpl programs.
Hence, we are only considering programs with one type of negation, which we will show, corresponds to strong negation in the \ASS.
Note that, as opposed to the \ASS, the WCS is defined in terms of the 
three-valued \mbox{\L ukasiewicz} logic instead of classical logic with strong negation.

\cite{hk:2009a} showed that \wcsmpl programs always have a unique \wcmodel which can be computed by the following consequence operator~\cite{stenning:vanlambalgen:2008}:\footnote{
The correspondence of this unique \wcmodel and the well-founded model~\cite{gelder:ross:schlipf:1991} for \wcsmpl programs without positive cycles, has been shown in~\cite{dhw:2014}.}
Given an interpretation~$I$ and a \wcsmpl program~$\CalP$,
the application of~$\Svl$ to $I$ and $\CalP$,
denoted by $\Svlp(I)$, is an interpretation $J = \langle J^\top, J^\bot \rangle$ defined as follows:
 \[\begin{array}{@{\hspace{0mm}}l@{\hspace{1mm}}c@{\hspace{1mm}}l@{\hspace{1mm}}l}    
J^\true & = & \{A \mid & \text{there is }A \leftarrow \Body \in \CalP
\text{ such that } 
I(\Body) = \true \}, \\
J^\bot & = & \{A \mid & \text{there is } A \leftarrow \Body \in  \CalP
\text{ and }
\\&&&  
\text{all } A \leftarrow \Body \in \CalP

\text{ satisfy } I(\Body) = \bot \}. 
   \end{array}
\]

The following example illustrates the WCS by means of two cases of Byrne's suppression task from the introduction.

\begin{example}\label{ex:wcs}
Let $\newprogram\label{prg:sup1}$ be the \wcsmpl program consisting of the rules in~\eqref{prg:suppression1} in
the introduction.
$\tranL{\WComp{\program\ref{prg:sup1}}}$ is the following \Ltheory:
\begin{gather*}
\lib \dimpliesL \essay \wedge \sneg \Ab_1
\hspace{2cm}
\ess \dimpliesL \bot
\hspace{2cm}
\Ab_1 \dimpliesL \bot
  \label{prg:suppression1.ltheory}
\end{gather*}
whose unique \wcmodel is
$\tuple{\emptyset,\set{\ess,\lib, \Ab_1}} = \set{ \sneg \ess, \sneg \lib, \sneg \Ab_1 }$.
This program illustrates why assumptions such as $\ess \leftarrow \bot$,
though being tautologies in \mbox{\L ukasiewicz} logic, are not tautologies under the WCS:
After the the weak completion transformation they become equivalences,
$\ess \dimpliesL \bot$, and, thus, $\ess$ has to be false.
Note that assumptions can also be overwritten by facts.
Let for instance 
$\newprogram\label{prg:sup1b} = \program\ref{prg:sup1} \cup \set{ \ess \leftarrow \top }$
be the program obtained by adding the fact $e$ to the above program.
Then, its weak completion  $\tranL{\WComp{\program\ref{prg:sup1b}}}$ is as follows:
\begin{gather*}
\lib \dimpliesL \essay \wedge \sneg \Ab_1
\hspace{2cm}
\ess \dimpliesL \bot \vee \top
\hspace{2cm}
\Ab_1 \dimpliesL \bot
\end{gather*}
As $\ess \dimpliesL \bot \vee \top \equivL \ess \dimpliesL \top$, the unique \wcmodel
of $\CalP_{\ref{prg:sup1b}}$ is $\tuple{\set{\ess, \lib},\set{\Ab_1}} = \set{ \ess, \lib, \sneg \Ab_1 }$, where $\ess$ and $\lib$ are true.
Let $\newprogram\label{prg:sup2}$ be the \wcsmpl program consisting of the rules in~\eqref{prg:suppression2} in the introduction.
$\tranL{\WComp{\program\ref{prg:sup2}}}$ is the following \Ltheory:
\begin{gather*}
\lib \dimpliesL (\essay \wedge \sneg \Ab_1) \vee (t \wedge \sneg \Ab_2)
\hspace{0.8cm}
\ess \dimpliesL \bot
\hspace{0.8cm}
\Ab_1 \dimpliesL \bot
\hspace{0.8cm}
\Ab_2 \dimpliesL \bot
  \label{prg:suppression2.ltheory}
\end{gather*}
whose unique \wcmodel is
$\tuple{\emptyset,\set{\ess,\Ab_1,\Ab_2}} = \set{ \sneg \ess, \sneg \Ab_1, \sneg \Ab_2 }$.
That is, $e$, $\Ab_1$ and $\Ab_2$ are false, while $\lib$ and $t$ are unknown.
\end{example}


\section{Correspondence between \ASS and WCS}\label{sec:correspondence}

Let us first discuss the main differences between both semantics according to the two examples of the suppression task in the introduction.

\begin{examplecont}{ex:wcs}\label{ex:wcs->asp}
Consider the \wcsmpl program $\CalP_{\ref{prg:sup1}}$: Its corresponding normal program can be obtained by replacing every assumption of the form~$A \leftarrow \bot$ 
in $\CalP_{\ref{prg:sup1}}$ by a fact with strong negation~$\sneg A$.
The resulting program, $\newprogram\label{prg:sup1.asp}$ consists of the following rules:
\begin{gather}
\lib \leftarrow \essay \wedge \sneg \Ab_1
\hspace{2cm}
\sneg \ess
\hspace{2cm}
\sneg \Ab_1 
  \label{prg:suppression1.asp}
\end{gather}
Its unique answer set is
$\tuple{\emptyset,\set{\ess,\Ab_1}} = \set{ \sneg \ess, \sneg \Ab_1}$,
which does not coincide with the \wcmodel of $\CalP_{\ref{prg:sup1}}$,
as $\lib$ is false under the WCS, but unknown under the \ASS.

\end{examplecont}
The above example illustrates that replacing assumptions by strong negation facts is not enough to obtain the same results between WCS and \ASS.

As mentioned previously the \ASS and WCS can be respectively defined in terms of classical logic with strong negation and 
three-valued {\L}ukasiewicz logic.
Interestingly, 
Vakarelov~\cite{vakarelov1977notes} showed that
there is a correspondence between \Llogic and \Nlogic, in the sense that all connectives of one logic are definable in 
terms of the other one.
In particular, here we are interested in translating from the WCS to the \ASS and, thus, that implies a translation from \Llogic to \Nlogic.
Formally, given a \Ltheory $\Gamma$, by $\GammaI{\Gamma}$ we denote the result of
replacing in~$\Gamma$
every occurrence of $\varphi \impliesL \psi$ by $(\varphi \impliesN \psi) \wedge (\sneg\psi \impliesN \sneg\varphi)$.

\begin{theorem}[Theorem~11 in~\protect\cite{vakarelov1977notes}]
\label{th:equivLlogicNlogic.vakarelov}\smallskip
Given any \Ltheory $\Gamma$,
an interpretation~$I$ is a model of $\Gamma$ under \Llogic iff $I$ is a model of $\tranN{\Gamma}$ under \Nlogic. 
\end{theorem}

Based on this result, we can establish the correspondence between the \ASS and the WCS.

%
%
We need rules that \textit{negatively complete} the 
information of the given program.
Let us now formalize this idea by defining the definition completion of a program.

\begin{definition}
Given a normal nested program~$\CalP$, its \emph{definition completion} is defined as follows:
\begin{eqnarray}
  \nComp{\CalP} \ \ &\eqdef& \ \ \CalP \cup 
  \{ \sneg L \leftarrow  \sneg \varphi \mid (L \leftarrow \varphi) \in \Pdisj{\CalP}
  \}
 \end{eqnarray}
\end{definition}

%

Let us apply the suggested characterization
for the programs in Example~\ref{ex:wcs->asp}.

\begin{examplecont}{ex:wcs->asp}\label{ex:wcs->asp2}
Given~\program\ref{prg:sup1.asp}, its definition completion is as follows:
\begin{gather*}
\nComp{\program\ref{prg:sup1.asp}} = \CalP_{\ref{prg:sup1.asp}} \cup \{ 
\hspace{0.25cm}
\sneg \lib \!\leftarrow\! \sneg\, (\essay \wedge \sneg \Ab_1),
\hspace{0.25cm}
\sneg\sneg \ess \!\leftarrow\! \sneg\top,
\hspace{0.25cm}
\sneg\sneg \Ab_1 \!\leftarrow\! \sneg\top
\hspace{0.25cm} \}
\end{gather*}
Note that $\sneg \lib \!\leftarrow\! (\sneg\, \essay \wedge \sneg \Ab_1)$ is equivalent to 
$\sneg \lib \leftarrow \sneg\essay \vee \Ab_1$,
while the last two rules are tautologies under the \ASS.
The unique answer set of
$\nComp{\program\ref{prg:sup1.asp}}$
corresponds to the unique \wcmodel of~$\program\ref{prg:sup1}$.

\end{examplecont}

\subsection{Characterization of WCS in Terms of \ASS}

We will now introduce some auxiliary results that will help us 
to show the correspondence between WCS and \ASS.
Let us start by showing that the answer sets of any program coincide with the answer sets of its weak completion.
The proof of this statement relies on the the following lemma which is a straightforward lifting of the Completion Lemma from~\cite[p. 23]{FerrarisL05} to the class of programs with strong negation.

\begin{lemma}\label{lemma:reg:completion}
Let~$\CalP$ be any program, let $\Atoms$ be any set of atoms (not necessarily equal to $\atoms(\CalP)$)
and let \mbox{$Q \subseteq (\Atoms \cup \sneg \Atoms)$} be any set of literals such that $Q \cap \Head(\CalP) = \emptyset$.
Let $\varphi_L$ be some implication-free formula for each literal $L \in Q$ and $I$ be an interpretation.
Then,
the following two statements are equivalent:
\begin{enumerate}
 \item $I$ is an answer set of $\CalP \cup \setm{ L \leftarrow \varphi_L}{ L \in Q }$.
 \item $I$ is an answer set of  $\CalP \cup \setm{ L \leftrightarrow \varphi_L}{ L \in Q }$.
\end{enumerate}
\end{lemma}

\begin{proposition}\label{prop:answersets.wc}
Given any normal nested program $\CalP$, an interpretation~$I$ is an answer set of $\CalP$ if and only if $I$ is an answer set of $\WComp{\CalP}$.
\end{proposition}
\begin{proof}
Assume that $\CalP$ is regular.
From  item~(ii) of Proposition~6 in~\cite{LifschitzTT99},
~$\CalP$ and $\Pdisj{\CalP}$ have the same answer sets.
Note that there is a unique rule with head $L$ for each literal $L \in (\Atoms \cup \sneg\Atoms)$ in $\Pdisj{\CalP}$.
Hence, $\WComp{\CalP}$ is obtained by replacing all occurrences of $\leftarrow$ in $\Pdisj{\CalP}$ by $\leftrightarrow$
and, from Lemma~\ref{lemma:reg:completion} (by taking $\CalP = \emptyset$), 
$\Pdisj{\CalP}$ and $\WComp{\CalP}$ have the same answer sets.
In case that $\CalP$ is not regular, we have that $I$ is an answer set of $\CalP$ iff $I$ is an answer set of 
$\reg{\CalP}$ iff $I$ is an answer set of $\WComp{\reg{\CalP}} = \reg{\WComp{\CalP}}$ iff $I$ is an answer set of $\WComp{\CalP}$.
\end{proof}

\begin{lemma}\label{lem:answer-set.Nlogic<->Llogic}
Given any normal nested logic program~$\CalP$, an interpretation $I$
is a model of $\tranN{\wComp{\nComp{\CalP}}}$ under \Nlogic 
if and only if $I$ is a model of $\tranL{\WComp{\CalP}}$.
\end{lemma}
\begin{proof}
Note that
$\tranN{\wComp{\nComp{\CalP}}}$
has a pair of equivalences of the form
\begin{align}
A &\dimpliesN \varphi_1 \vee \cdots  \vee \varphi_n
  \label{eq:1:cor:answer-set.Nlogic}
\\
\sneg A &\dimpliesN \sneg\varphi_1 \wedge \cdots \wedge \sneg\varphi_n
  \label{eq:2:cor:answer-set.Nlogic}
\end{align}
for each $A \in \Atoms$ with
\mbox{$\Def(\CalP,L)  = \set{ A \leftarrow \varphi_1, \dotsc, A \leftarrow \varphi_n }\neq \emptyset$}.
On the other hand, we have that
$\tranL{\wComp{\CalP}}$
has an equivalences of the form
\begin{align}
 A &\dimpliesL \varphi_1 \vee \cdots \vee \varphi_n
  \label{eq:3:cor:answer-set.Nlogic}
\end{align}
for each atom $A \in \Atoms$ with
\mbox{$\Def(\CalP,L)   = \set{ A \leftarrow \varphi_1, \dotsc, A \leftarrow \varphi_n }\neq \emptyset$}.
By Theorem~\ref{th:equivLlogicNlogic.vakarelov},
the models of $\tranL{\wComp{\CalP}}$ under \Llogic
and the models of $\tranN{\tranL{\wComp{\CalP}}}$ under \Nlogic coincide.
Note now that, by definition, we have that~\eqref{eq:3:cor:answer-set.Nlogic} is equivalent to the following formula
\begin{align}
(A \impliesL \varphi_1 \vee \dotsc \vee \varphi_n) \wedge (\varphi_1 \vee \dotsc \vee \varphi_n \impliesL A)
\end{align}
Hence,
$\tranN{\tranL{\wComp{\CalP}}}$
contains a formula of the form $\psi^A_1 \wedge \psi^A_2 \wedge \psi^A_3 \wedge \psi^A_4$
for each atom $A \in \Atoms$ with
\mbox{$\Def(\CalP,A)  = \set{ A \leftarrow \varphi_1, \dotsc, A \leftarrow \varphi_n }\neq \emptyset$}
where
\begin{align}
 \psi^A_1 \eqdef& A \impliesN \varphi_1 \vee \cdots \vee \varphi_n
\\
\psi^A_2 \eqdef& \sneg A \impliesN \sneg(\varphi_1 \wedge \cdots \wedge \varphi_n)
\\
\psi^A_3 \eqdef& \varphi_1 \vee \cdots \vee \varphi_n \impliesN A
\\
\psi^A_4 \eqdef& \sneg(\varphi_1 \wedge \cdots \wedge \varphi_n) \impliesN \sneg A
\end{align}
Note that, by definition, $\psi^A_1 \wedge \psi^A_3$ is equivalent to~\eqref{eq:1:cor:answer-set.Nlogic}.
Besides $\psi^A_3 \wedge \psi^A_4$ can be equivalently rewritten as
\begin{align}
\sneg A \dimpliesN \sneg\,(\varphi_1 \vee \dotsc \vee \varphi_n)
\end{align}
that is equivalent to~\eqref{eq:2:cor:answer-set.Nlogic}, i.e.\
${\tranL{\wComp{\CalP}}}$
and
$\tranN{\wComp{\nComp{\CalP}}}$ have the same models.
\end{proof}

\begin{proposition}{\label{thmr:answer-set.wc}}
Given any normal nested program~$\CalP$
and interpretation $I$, if $I$ is an answer set of $\nComp{\CalP}$, then
$I$ is a model of $\tranL{\WComp{\CalP}}$.
\end{proposition}
\begin{proof}
From Lemma~\ref{lem:answer-set.Nlogic<->Llogic},
$I$ is a model of $\tranN{\wComp{\nComp{\CalP}}}$
iff $I$ is a model of $\tranL{\WComp{\CalP}}$.
From Proposition~\ref{prop:answersets.wc},
the answer sets of $\nComp{\CalP}$ and $\wComp{\nComp{\CalP}}$ are the same.
Furthermore, from Proposition~\ref{prop:answer.set->closedNmodel.basic}, the answer sets of
$\wComp{\nComp{\CalP}}$
are models of
$\tranN{\wComp{\nComp{\CalP}}}$ under \Nlogic.
Hence, the answer sets of $\nComp{\CalP}$ are \wcmodels of~$\CalP$.
\end{proof}

Given Lemma~\ref{lem:answer-set.Nlogic<->Llogic}, Proposition~\ref{prop:unique.answer.set<->closedleastNmodel},\ref{prop:answersets.wc} and~\ref{thmr:answer-set.wc}, we can now show 
how the definition completion of a program precisely characterizes the WCS in terms of the \ASS.

\begin{theorem}{\label{thmr:answer-set.wc.basic}}
Given any \wcsmpl program~$\CalP$ and interpretation~$I$,
the following two statements are equivalent:
\begin{enumerate}
\item $I$ is the unique \wcmodel of $\CalP$,
\item $I$ is the unique answer set of~$\nComp{\CalP}$.
\end{enumerate}
The following two statements are also equivalent:
\begin{enumerate}
\item $\CalP$ has no \wcmodel,
\item $\nComp{\CalP}$ has no answer set.
\end{enumerate}
\end{theorem}

\begin{proof}
Assume that $\tranN{\wComp{\nComp{\CalP}}}$ has no model.
From Lemma~\ref{lem:answer-set.Nlogic<->Llogic},
it follows that
$\tranL{\WComp{\CalP}}$
has no model either and, thus, there is no \wcmodel of $\CalP$.
Besides, from Proposition~\ref{thmr:answer-set.wc},
the lack of model of $\tranL{\WComp{\CalP}}$ also implies that~$\nComp{\CalP}$ has no answer set.
Otherwise, $\tranN{\wComp{\nComp{\CalP}}}$ has a model
and, since is a \wcsmpl program and thus basic, from
Proposition~\ref{prop:unique.answer.set<->closedleastNmodel},
we get that there is an interpretation $I$ which the \mbox{$\subseteq$-least} model of $\tranN{\wComp{\nComp{\CalP}}}$ and, thus, the unique answer set of $\wComp{\nComp{\CalP}}$.
From Proposition~\ref{prop:answersets.wc} this implies that
$I$ is the unique answer set of $\nComp{\CalP}$.
Furthermore, from Lemma~\ref{lem:answer-set.Nlogic<->Llogic}, this also implies that
it is the $\subseteq$-least model of $\tranL{\WComp{\CalP}}$
and, thus, the unique \wcmodel of~$\CalP$.
\end{proof}


\section{Conclusions and Future Work}

We have shown how logic programs under the Weak Completion Semantics can be translated into logic programs under the Answer Set Semantics by using the \emph{definition completion}.
This completion adds rules supporting the strong negation of a defined atom whenever all the bodies of all rules defining it are false.
This transformation has been illustrated by two examples of Byrne's suppression task.

This result allows us to use all the knowledge representation features of Answer Set Programming, including \emph{default negation}, in combination with this completion
and opens two interesting future possibilities: On the one hand, in~\cite{dhp:2017}, logic programs under the Weak Completion Semantics were extended with a context operator to capture negation as failure.
Hence, an immediate question is whether these \emph{contextual logic programs} can also be translated into logic programs under the Answer Set Semantics by using default negation.
On the other hand, it would be interesting to investigate how the twelve cases of the suppression
task could be represented 
by means of default negation or the context operator in order to ensure \emph{elaboration tolerance}~\cite{mccarthy1998elaboration}.

Another interesting observation is that the proof of the correspondence between these two semantics relies on the use of strong negation as a connective in its own right, as opposed to the usual convention of considering that strong negation can only be applied to atoms.
This extension was first considered by David Pearce in~\cite{pearce:1997}. It would be worth to investigate
how the usual properties of the Answer Set Semantics can be extended to this class of programs and how its use can ease other knowledge representation problems.

\paragraph{Acknowledgements.}
We are thankful to David Pearce for pointing to Vakarelov's work on the relation between {\L}ukasiewicz logic and classical logic with strong negation.

\bibliographystyle{splncs03}
\bibliography{bib}
\end{document}